\documentclass{article}

\usepackage{amsthm}
\usepackage{graphicx}
\usepackage{lineno,hyperref}
\usepackage{times}
\usepackage{algorithmic}
\usepackage{algorithm}
\usepackage{amssymb}
\usepackage{amsmath}
\usepackage{amsfonts}
\usepackage{nccmath}
\usepackage{graphicx}
\usepackage{url}
\usepackage{multirow}
\usepackage{wrapfig}
\usepackage{caption}
\usepackage{caption}
\usepackage{color}
\usepackage{times}
\usepackage{fixmath}
\usepackage{bm}

\newtheorem{definition}{Definition}

\newtheorem{lemma}{Lemma}
\newtheorem{theorem}{Theorem}
\newtheorem{corollary}{Corollary}

\begin{document}

\title{Sublinear Update Time Randomized Algorithms for
Dynamic Graph Regression}

\author{Mostafa Haghir Chehreghani \\
              Department of Computer Engineering \\
              Amirkabir University of Technology (Tehran Polytechnic), Iran \\
               mostafa.chehreghani@aut.ac.ir          
}

\date{}
\maketitle

\begin{abstract}
A well-known problem in data science and machine learning
is {\em linear regression}, which is recently
extended to dynamic graphs.
Existing exact algorithms for updating the solution
of dynamic graph regression require
at least a linear time
(in terms of $n$: the size of the graph).
However, this time complexity might be
intractable in practice.

In the current paper, we utilize
{\em subsampled randomized Hadamard transform}
and \textsf{CountSketch}
to propose the first sublinear update time
randomized algorithms for regression of general dynamic graphs.
Suppose that we are given a $n\times d$
matrix embedding $\mathbold M$
of the graph,
where $d \ll n$ and $\mathbold M$ has certain properties.
Let $r$ be the number of samples required
by subsampled randomized Hadamard transform
for a $1\pm \epsilon$ approximation,
which is a sublinear of $n$.
Our first algorithm
supports
edge insertion and edge deletion and
updates the
approximate solution
in $O(rd)$ time.
Our second algorithm
is based on \textsf{CountSketch}
and supports
edge insertion, edge deletion,
node insertion and node deletion.
It updates the approximate solution
in $O(qd)$ time,
where $q=O\left(\frac{d^2}{\epsilon^2} \log^6(d/\epsilon) \right)$.
\end{abstract}




\section{Introduction}
\label{sec:introduction}

One of the well-studied machine learning problems is
{\em linear regression},
which is traditionally defined as follows.
We receive $n$ data,
where for each $i \in [1,n]$,
the data consists of
a row in a matrix $\mathbold{A}$ and
a single element in a vector $\mathbold{b}$.
Matrix $\mathbold{A}$ is called {\em predictor values}
and $\mathbold{b}$ is called {\em measured values}.
The goal is to find a vector $\mathbold{x}$ such that
$\mathbold{A} \cdot \mathbold{x}$ is the closest point to $\mathbold{b}$ in the column span of $\mathbold{A}$,
under some distance measure, e.g.,
the Euclidean distance
(which is also called the least squares distance
or the $L2$ norm).
In other words, we want to solve the following problem:
\[argmin_{\mathbold{x}} ||\mathbold{A} \cdot \mathbold{x} - \mathbold{b}||_2,\]
or the equivalent problem:
\begin{equation}
\label{eq:regression}
argmin_{\mathbold{x}} ||\mathbold{A} \cdot \mathbold{x} - \mathbold{b}||_2^2.
\end{equation}

There is a long history of research on the
regression problem for static matrix data
and graph data \cite{ecf352e8da75424b8b40be6f22ade69d}.
All the current best bounds for graph regression are
by using a matrix embedding (representation)
for the graph.
As an example motivating (linear) graph regression using a matrix embedding for the graph,
assume that we are
given the graph of a friendship network in
which a score is assigned to each node
reflecting its reputation or weight or importance etc.
Suppose that we want to find a function with a minimum error for the scores of the nodes, which is linear in terms of their structural
properties. Hence, first we need to find a matrix embedding $\mathbold M$ of the nodes, in which each row $i$ represents the structural properties
of node $i$.
Then we need to find a function for the scores
which is linear in terms of
the values in the rows of $\mathbold M$.

Since most of real-world graphs are {\em dynamic},
recently the problem was extended to
{\em dynamic graphs} \cite{DBLP:journals/corr/abs-1903-10699,DBLP:conf/stoc/DurfeeGGP19}.
Dynamic graphs are graphs that change over time by
a sequence of update operations.
They are generated in many domains such as
the world wide web, social and information networks, technology networks and communication networks.
An update operation in a graph might be
either an edge insertion or an edge deletion or
a node insertion or a node deletion.

Given a $n \times d$
(update-efficient) matrix embedding\footnote{Note that
this notion of {\em embedding} is different from the notion of {\em embedding} used in {\em graph pattern mining} \cite{DBLP:journals/fgcs/ChehreghaniABB20,DBLP:journals/datamine/ChehreghaniB16,DBLP:journals/tsmc/ChehreghaniCLR11}.}
of a graph $G$,
the author of \cite{DBLP:journals/corr/abs-1903-10699}
proposed an exact algorithm for dynamic graph regression,
wherein first an
$O\left( \min \left\{ nd^2,n^2d \right\} \right)$
time pre-processing is performed.
Then after any update operation in the graph,
the solution is updated in $O(nd)$ time.
However,
since in most of applications $n$ is a very large
quantity, this time complexity might be
too high to be used in practice.
Therefore,
we are interested in
developing algorithms that are considerably faster
than the exact algorithm,
at the expense of producing an approximate solution.
In particular, we want to develop algorithms
that have a {\em sublinear update time},
in terms of $n$.

To do so,
in the current paper
we utilize two sketching techniques, namely
{\em subsampled randomized
Hadamard transform} \cite{DBLP:conf/soda/AilonL08}
and \textsf{CountSketch} \cite{Clarkson:2017:LAR:3038256.3019134},
to develop sublinear update time randomized algorithms
for the dynamic graph regression problem:
\begin{itemize}
\item{(Theorem~\ref{theorem:phd_edge} and Corollary~\ref{corollary:phd}).}
Let $r$ be a quantity that indicates
the number of samples required
for a $1 \pm \epsilon$ approximation,
as defined in Equations~\ref{eq:r} and \ref{eq:r2}
of Theorem~\ref{theorem:phd}.
Our first randomized algorithm
is based on subsampled randomized
Hadamard transform and
supports edge insertion and edge deletion.
It updates the approximate solution
in $O(rd)$ time.
With $d \ll n$ and in particular considering
$d$ as a constant,
this yields a sublinear update time.
\item
(Theorems~\ref{theorem:countsketch_edge}, \ref{theorem_node_insertion} and \ref{theorem:node_deletion},
and Corollary~\ref{corollary:count_sketch}).
Let $q=O\left(\frac{d^2}{\epsilon^2}
\log^6(d/\epsilon)\right)$ be the number of
samples required for a $1 \pm \epsilon$ approximation,
using \textsf{CountSketch}.
Our second randomized algorithm uses
\textsf{CountSketch} and
updates the approximate solution
in $O(qd)$ time.
Therefore if $d$ and $\epsilon$ are considered as constants,
it yields a
constant update time randomized algorithm.
Unlike our first algorithm,
our second algorithm supports all the update operations
edge insertion, edge deletion,
node insertion and node deletion.
\end{itemize}
Note that subsampled randomized Hadamard transform
and \textsf{CountSketch}
have already been used to improve regression
in static data \cite{DBLP:conf/soda/AilonL08,Drineas2011,doi:10.1137/120874540,Clarkson:2017:LAR:3038256.3019134}.
However, in this paper for the first time
we show how they can be used to improve update time
in a dynamic setting,
where it is required to update
the sketches and the approximate solution,
after an update operation in the data.

While our randomized algorithms considerably improve
update time upon the exact algorithm,
we also analyze their relative performance.
We show that
under some assumptions,
if $\ln n < \epsilon^{-1}$ our first algorithm
outperforms our second algorithm
and if $\ln n \geq \epsilon^{-1}$ our second algorithm reveals a better
update time.

The rest of this paper is organized as follows.
In Section~\ref{sec:preliminaries},
we present preliminaries and necessary background
and definitions used in the paper.
In Section~\ref{sec:relatedwork},
we provide an overview on related work.
In Section~\ref{section:sketch}, we briefly introduce
subsampled randomized Hadamard transform and \textsf{CountSketch}.
In Section~\ref{section:randomized},
we present our first randomized algorithm for the dynamic graph regression problem, which is based on subsampled randomized Hadamard transform\footnote{Parts of the results discussed in Section~\ref{section:randomized} were presented
in {\em Proceedings of the 29th ACM International Conference on Information and Knowledge Management} (CIKM~2020), pp. 2045-2048~\cite{DBLP:conf/cikm/Chehreghani20}.}.
In Section~\ref{section:countsketch},
we introduce our second randomized algorithm, which is based on
\textsf{CountSketch}.
We discuss and compare our proposed algorithms in
Section~\ref{sec:discussion}.
Finally, the paper is concluded
in Section~\ref{sec:conclusion}.

\section{Preliminaries}
\label{sec:preliminaries}

In this paper,
we use the following standard for notations and symbols:
lowercase letters for scalars,
uppercase letters for constants and graphs,
bold lowercase letters for vectors
and bold uppercase letters for matrices.
By $G$ we refer to a graph
which is simple and unweighted.
We use $n$ to denote the number of nodes of $G$.
We define a {\em dynamic graph} as a graph that
changes over time by a sequence of {\em update operations}.
The {\em adjacency matrix} of $G$ is a square
$n \times n$ matrix 
such that
its element in row $i$ column $j$
is $1$
iff there exists an edge
from node $i$ to node $j$
(and $0$ if there is no such an edge).
We define the {\em distance} between
node $u$ and node $v$,
denoted by $dist(u,v)$,
as the size, i.e., the number of edges of a shortest path
connecting $u$ to $v$.

Let $\mathbold{A} \in \mathbb{R}^{n \times d}$.
The {\em rank} of $\mathbold{A}$ is defined as
the maximum number of its linearly independent column vectors.
The {\em transpose} of $\mathbold{A}$,
denoted with $\mathbold{A}^*$,
is defined as
an operator that switches the row and
column indices of $\mathbold{A}$.
The Singular Value Decomposition (SVD) of a $n \times d$ matrix $\mathbold{A}$
is defined as
$\mathbold{U} \cdot \mathbold{\Sigma} \cdot \mathbold{V}^*$,
where $\mathbold{U}$ is a $n \times d$ matrix with orthonormal columns,
$\mathbold{\Sigma}$ is a $d \times d$ diagonal matrix with
non-negative non-increasing entries down the diagonal,
and $\mathbold{V}^*$ is a $d \times d$ matrix with orthonormal rows.
The Euclidean norm or $L_2$ norm of a vector $\mathbold{x}$ of size $n$,
denoted with $||\mathbold{x}||_2$, is defined as
$\sqrt{\mathbold{x}_1^2+\cdots+\mathbold{x}_n^2}$.

The {\em Moore-Penrose pseudoinverse} of
matrix $\mathbold{A} = \mathbold{U} \cdot \mathbold{\Sigma} \cdot \mathbold{V}^*$,
denoted with $\mathbold{A}^\dag$,
is the $d \times n$ matrix $\mathbold{V} \cdot \mathbold{\Sigma}^{\dag} \cdot \mathbold{U}^*$,
where $\mathbold{\Sigma}^{\dag}$ is a $d \times d$ diagonal matrix
defined as follows:
$\mathbold{\Sigma}^{\dag}[i,i] = 1/\mathbold{\Sigma}[i,i] $, if $\mathbold{\Sigma}[i,i]>0$
and $0$ otherwise.
It is well-known that the solution
\begin{equation}
\label{eq:solution}
\mathbold{x} = \mathbold{A}^{\dag} \cdot \mathbold{b}
\end{equation}
is an optimal solution for Equation~\ref{eq:regression}
and it has minimum $L2$ norm \cite{general_inverse_book}.

The approximate version of the regression problem
is defined as

\begin{equation}
\label{eq:approximate_regression}
argmin_{\mathbold{x'}} || \mathbold A \cdot \mathbold{x'} - \mathbold b ||_2^2 =
(1 \pm \epsilon) argmin_{\mathbold{x}} || \mathbold A \cdot \mathbold x - \mathbold b ||_2^2,
\end{equation}
where $\mathbold{x}$ is the optimal solution,
defined in Equation~\ref{eq:solution},
and $\epsilon \in (0,1)$
defines the desired accuracy.
As we will see in Section~\ref{section:sketch},
sketching techniques
can be used to solve this approximate version.

\section{Related work}
\label{sec:relatedwork}

In recent years, a number of algorithms have been proposed for different
learning problems over the entire
graphs \cite{DBLP:journals/ida/ChehreghaniRLC07,DBLP:journals/ml/SaigoNKKT09,DBLP:journals/ida/ChehreghaniCLRG09}
or nodes of a graph \cite{DBLP:journals/jacm/KleinbergT02,DBLP:conf/nips/HerbsterP06,DBLP:journals/corr/abs-1903-10699}.
Kleinberg and Tardos \cite{DBLP:journals/jacm/KleinbergT02} studied the classification problem
for nodes of a static graph and showed
the connection of their general formulation to Markov random fields.
Herbster and Pontil
\cite{DBLP:conf/nips/HerbsterP06} studied the problem of online label prediction of
a graph with the perceptron.
The key difference between {\em online setting}
\cite{DBLP:conf/icml/HerbsterPW05,DBLP:conf/nips/HerbsterLP08,DBLP:conf/colt/HerbsterL09,DBLP:journals/jmlr/HerbsterPP15}
and {\em dynamic setting} is
that {\em online setting} is used when it is computationally infeasible to solve the learning problem over the entire dataset.
However, in {\em dynamic setting} the learning problem can be solved over the entire dataset and the challenge is to efficiently update
the solution when the dataset changes.
Culp, Michailidis and Johnson
\cite{On_Multi_view_Learning}
presented representative multi-dimensional
view smoothers on graphs that are based on graph-based transductive learning \cite{Zhu2008}.
The authors of
\cite{Manifold_Regularization} proposed a family of learning algorithms based on a new form of regularization so that some of transductive graph
learning algorithms can be obtained as special cases.
Kovac and Smith
\cite{ecf352e8da75424b8b40be6f22ade69d} extended
a model for nonparametric regression
of nodes of a static graph, where
the distance between estimate and observation is measured
by $L_2$ norm.
Chehreghani~\cite{DBLP:journals/corr/abs-1903-10699}
studied regression over dynamic graphs.
He proposed an exact algorithm for updating
the optimal solution of the problem,
where the update time is at least linear
in terms of the number of nodes.
In the current paper, we present randomized
algorithms with sublinear update times.

A research problem that
may have some connection to our studied problem
is learning embeddings or representations for nodes
of a graph \cite{DBLP:conf/kdd/GroverL16},
\cite{DBLP:conf/icml/YangCS16}, \cite{DBLP:conf/icml/NiepertAK16}. While this problem has become
more attractive in recent years,
it dates back to several decades ago.
For example,
Parsons and Pisanski
\cite{PARSONS1989143} presented vector embeddings for nodes of a graph such that
the inner product of the vector embeddings of any two nodes $i$ and $j$ is negative iff
$i$ and $j$ are connected by an edge;
and it is $0$ otherwise.

In the literature, there also exist several
{\em updating} algorithms
for different problems over dynamic graphs.
Durfee et.al.~\cite{DBLP:journals/corr/abs-1804-04038} presented a sublinear update time randomized algorithm to
approximate {\em effective resistances}.
The effective resistance between two nodes
is the electrical resistance
seen between the nodes of a resistor network
where edge weights form conductances.
Their algorithm supports
edge insertion/deletion
and yields a
$1\pm \epsilon$ approximation
with a probability at least
$1- 1/poly(n)$.
Durfee et.al.~\cite{DBLP:conf/stoc/DurfeeGGP19}
gave
algorithms for updating Schur complements of general graphs,
that support edge insertion/deletion and
node insertion.
Their algorithms maintain at any
time a $1\pm \epsilon$ approximation to the Schur complement.
The authors also presented a sublinear update time algorithm
for least squares regression
of bounded-degree graphs with gradual changes in $\bm b$.
However in the current paper,
we present the first sublinear update time algorithms
for the general class of graphs with unbounded degrees,
that captures most important real-world networks.
Chen et.al.~\cite{DBLP:journals/corr/abs-2005-02368}
developed a technique to reduce
optimization problems based on undirected graphs
to finding a data-structure notion of node sparsifiers.
Using this technique,
they presented a sublinear update time algorithm for flows.
An overview on a number of update algorithms
for different machine learning problems can be
found in~\cite{doi.org/10.1002/widm.1393}.

\section{Sketching techniques}
\label{section:sketch}

In this section, we briefly describe
{\em subsampled randomized Hadamard transform} and
\textsf{CountSketch}.
Let $\mathbold A$ be a $n\times d$ matrix.
A subsampled randomized Hadamard transform
for $\mathbold A$ is defined as
$\mathbold P \cdot \mathbold H \cdot \mathbold D$,
where
\begin{itemize}
\item
matrix $\mathbold D$ is a $n \times n$ diagonal matrix with
$\pm 1$ on the diagonal
(each one with the same probability),
\item
matrix $\mathbold H$ is a $n \times n$ Hadamard matrix,
and 
\item
matrix $\mathbold P$
is a $r \times n$ matrix that
samples $r$ rows of $\mathbold P \cdot \mathbold H$
uniformly with replacement.
If in the $j^{th}$ sample row $i$ is selected,
$\mathbold P[j,i] = \frac{\sqrt{n}}{\sqrt{r}}$;
otherwise, it is $0$.
\end{itemize}

For $n=2^k$, the $n \times n$ Hadamard matrix
$\mathbold H$ is defined as follows:
$\mathbold H[i,j] = \frac{(-1)^{\langle i,j\rangle}}{\sqrt{n}},$
where $\langle i,j\rangle$ is the dot product of
the binary representations of $i$ and $j$ over the field $\mathbb F_2$.

A \textsf{CountSketch}
for the $n \times d$ matrix $\mathbold A$ is
a $q \times n$ matrix $\mathbold S$
($q$ is defined in Theorem~\ref{theorem:countsketch}),
defined as follows:
for every column, a single nonzero entry
is chosen uniformly at random,
which takes values $\pm 1$ with equal
probability \cite{Clarkson:2017:LAR:3038256.3019134}.
Therefore, $\mathbold S$ is a sparse matrix
which has only $n$ nonzero elements.
Moreover, $\mathbold S \cdot \mathbold A$
can be computed in a time proportional
to the number of nonzero elements
of $\mathbold A$ \cite{Clarkson:2017:LAR:3038256.3019134}.

The high level procedure of solving regression using
sketching (either subsampled randomized Hadamard transform
or \textsf{CountSketch}) is as follows:

\begin{itemize}
\item
Compute a sketching matrix $\mathbold S$
(either a $\mathbold P \cdot \mathbold H \cdot \mathbold D$ matrix
or a \textsf{CountSketch} matrix),
\item
Compute matrices $\mathbold S \cdot \mathbold A$
and $\mathbold S \cdot \mathbold b$,
\item
Compute and output the solution of the equation
\begin{equation}
\label{eq:approximate}
argmin_\mathbold{x'} || (\mathbold S \cdot \mathbold A) \cdot \mathbold{x'} -
\mathbold S \cdot \mathbold b||^2_2.
\end{equation}
\end{itemize}

The solution of Equation~\ref{eq:approximate} is
\begin{equation}
\label{eq:approximate_solution}
(\mathbold S \cdot \mathbold A)^{\dag}
\cdot \mathbold S \cdot \mathbold b,
\end{equation}
which we call the {\em approximate solution}.
When $\mathbold S$ is defined as a
$\mathbold P \cdot \mathbold H \cdot \mathbold D$ matrix,
Theorem~\ref{theorem:phd} states
the number of samples
(the number of rows of $\mathbold P$)
that are sufficient for producing
a $1\pm\epsilon$ approximation to the optimal solution.

\begin{theorem}[Theorem~2 (and the remark afterwards) of
\cite{Drineas2011}]
\label{theorem:phd}
Suppose $\mathbold A \in \mathbb R^{n \times d}$,
$\mathbold b \in \mathbb R^{n}$,
and let $\epsilon \in (0,1)$.
If
\begin{equation}
\label{eq:r}
r = \max \left\{ 48^2 d \ln(40nd) \ln\left(100^2d \ln(40nd)\right),
40d\ln(40nd)/{\epsilon} \right\},
\end{equation}
with a probability at least $0.8$, we have:
\begin{equation}
argmin_{\mathbold x'}|| \mathbold S \cdot \mathbold A \cdot \mathbold x' - \mathbold S \cdot \mathbold b||
\leq (1+\epsilon)
argmin_{\mathbold x} || \mathbold A \cdot \mathbold x - \mathbold b||.
\end{equation}
Time complexity of computing optimal
${\mathbold x'}$,
i.e., the approximate solution, is
\begin{equation}
\label{eq:time_complexity}
n(d + 1) + 2n(d + 1) \log_2(r + 1) + O(rd^2).
\end{equation}
In particular, assuming that $d \leq n \leq e^d$, we get:
\begin{equation}
\label{eq:r2}
r=O\left( d \ln d \ln n + \frac{d \ln n}{\epsilon}\right)
\end{equation}
and the time complexity becomes:
\begin{equation}
O\left( nd \ln \frac{d}{\epsilon} +
d^3 \ln d \ln n
+ \frac{d^3 \ln n}{\epsilon}
\right).
\end{equation}
\end{theorem}

When $\mathbold S$ is defined as a \textsf{CountSketch} matrix,
Theorem~\ref{theorem:countsketch} expresses
time complexity of the procedure
of computing a $1\pm\epsilon$ approximation
to the optimal solution.

\begin{theorem}[Theorem~30 of
\cite{Clarkson:2017:LAR:3038256.3019134}]
\label{theorem:countsketch}
Suppose that $\mathbold A \in \mathbb R^{n \times d}$,
$\mathbold b \in \mathbb R^{n}$
and $\epsilon \in (0,1)$.
Using a $q \times n$ \textsf{CountSketch} with
\begin{equation}
\label{eq:count_sketch_sample_no}
q=O \left(\frac{d^2}{\epsilon^2} \log^6 (d/\epsilon)\right),
\end{equation}
a $1\pm\epsilon$ approximation to the optimal solution
of linear regression
over $\mathbold A$ and $\mathbold b$
can be solved
with a probability at least $2/3$ in
\begin{equation}
O \left( nnz(\mathbold A) +
d^3 \epsilon^{-2} \log^7 (d/\epsilon)
\right)
\end{equation}
time, where $nnz(\mathbold A)$ is the number of nonzero
elements of $\mathbold A$.
\end{theorem}

\section{Dynamic graph regression using
subsampled randomized Hadamard transform}
\label{section:randomized}

In this section,
we utilize subsampled randomized Hadamard transform
to improve update time of dynamic graph regression,
at the cost of having
a $1\pm\epsilon$ approximation to the optimal solution.
We here restrict ourselves to
the following update operations:
i) {\em edge deletion},
wherein an edge is deleted from the graph, and
ii) {\em edge insertion},
wherein an edge is inserted between two nodes of the graph.
We refer to these operations as {\em edge-related}
update operations.
The reason that in this section
we do not consider node insertion and node deletion
is that as we will see later,
they require to change (the size of) the
used Hadamard matrix $\mathbold H$,
which requires $\Theta(n)$ time.
Hence and since we are looking for algorithms
that have a {\em sublinear} update time,
we do not consider these two operations.\footnote{
Moreover, a property of real-world graphs
is {\em densification}
\cite{DBLP:journals/tkdd/LeskovecKF07}, i.e.,
their number of edges grows superlinearly
in the number of their nodes.
Therefore, we may say that most of update operations
in a dynamic graph are related to edges,
rather than to nodes.
As a result, proposing algorithms
that are efficient for edge-related update operations
is useful and worthwhile.
For node insertions/deletions,
we may compute the solution from scratch,
whose time complexity is not much worse
than linear in $n$ (see Equation~\ref{eq:phd_pre_time}
of Corollary~\ref{corollary:phd}).
}

Before starting our proofs (and algorithms),
we note two intrinsic limitation of
randomized Hadamard transform:
i) $\mathbold H$ (respectively graph $G$)
must have a power of $2$ rows/columns
(respectively nodes),
ii) the matrix embedding $\mathbold M$
must have full rank.
For now on,
we forget these
two limitations.
We get back to them in Section~\ref{sec:limitations}.

We assume that the graph $G$
has an {\em edge-update-efficient} matrix embedding
$\mathbold M$, and we define
the regression problem with respect to it.
More precisely, we want to compute and update
$(\mathbold S \cdot \mathbold M)^{\dag}
\cdot \mathbold S \cdot \mathbold b$,
where $\mathbold M$ is {\em edge-update-efficient}.
{\em Edge-update-efficient} matrix embeddings
are a superset of {\em update-efficient}
matrix embeddings presented
in \cite{DBLP:journals/corr/abs-1903-10699}.
The class of {\em update-efficient} embeddings
characterizes those matrix embeddings for which
the optimal solution
of the graph regression problem can be
updated efficiently \cite{DBLP:journals/corr/abs-1903-10699}.
For example, adjacency matrix of $G$
belongs to this class.
{\em Edge-update-efficient} matrix embeddings,
defined in Definition~\ref{def:update_efficient},
characterize  those matrix embeddings for which
the approximate solution can be updated efficiently,
when the updated operation is
{\em edge-related}.

\begin{definition}
\label{def:update_efficient}
Let $\mathbold{M}$ be a $n \times d$ matrix embedding of a graph $G$
and $f$ be a complexity function.
We say $\mathbold{M}$ is {\em edge-update-efficient}$_{f}$,
if it satisfies the following condition:
if $\mathbold{M}$ and $\mathbold{M'}$
are the correct matrix embeddings before and after
one of the edge-related update operations,
there exist at most $K$ pairs of vectors $\mathbold{c^k}$ and $\mathbold{d^k}$,
with $K$ as a constant, such that:
$$\mathbold{M'}=\mathbold{M}+\sum_{k=1}^K \left( \mathbold{c^k} \cdot \mathbold{d^k}^* \right).$$
Each vector $\mathbold{d^k}$ has size $d$
and each vector $\mathbold{c^k}$
has size $n$ wherein only one entry,
whose position is known,
is nonzero\footnote{So, $\mathbold{c^k}$ can be compactly
stored by keeping only the position and the value
of its nonzero entry.}.
We refer to each pair $\mathbold{c^k}$ and $\mathbold{d^k}$ as a pair of {\em update vectors},
and to $\sum_{k=1}^K \left( \mathbold{c^k} \cdot \mathbold{d^k}^* \right)$ as the {\em update matrix}.
Also, it is feasible to compute
all pairs of update vectors in $O(f)$ time.
When function $f$ is clear from the context,
we drop it.
\end{definition}


At the high level,
our algorithm
consists of two phases:
the {\em pre-processing} phase
wherein we assume that we are given a static graph
and we find an approximate solution for it,
and the {\em update} phase,
wherein after an edge-related update operation in $G$,
the already found approximate solution
is revised to become valid for the new graph.
During pre-processing,
first we generate some
matrices $\mathbold P$, $\mathbold H$ and $\mathbold D$,
as defined in Section~\ref{section:sketch}.
Then we calculate
$\mathbold M' = \mathbold P \cdot \mathbold H \cdot \mathbold D \cdot \mathbold M$.
Then, we compute $\mathbold b' = \mathbold P \cdot \mathbold H \cdot \mathbold D \cdot \mathbold b$.
Then, we compute ${\mathbold M'}^{\dag}$
and finally, we compute ${\mathbold M'}^{\dag} \cdot \mathbold b'$.
Time complexity of the algorithm is stated in
Theorem~\ref{theorem:phd}.
In the following, first in Section~\ref{sec:phd_edge}
we discuss how
the approximate solution can be updated,
after an edge-related operation.
Then, in Section~\ref{sec:limitations} we discuss
how the limitations of
the used technique can be addressed.
The presented proofs are constructive.

\subsection{The update algorithm}
\label{sec:phd_edge}

In this section, we assume that the update operation is
an edge-related operation and
show that the approximate solution,
i.e., the value depicted in
Equation~\ref{eq:approximate_solution},
can be updated in $O(rd)$ time.
Here, we condition on the existence of an
{\em edge-update-efficient} matrix embedding,
without emphasizing any specific one.
Later in Section~\ref{sec:limitations},
we show that this condition holds.

\begin{theorem}
\label{theorem:phd_edge}
Let $\mathbold M$ be a $n \times d$
{\em edge-update-efficient} matrix embedding of graph $G$.
Suppose that
using a $r\times n$
subsampled randomized Hadamard transform
$\mathbold S$,
a $1\pm\epsilon$ approximation to the optimal solution of
graph regression over $G$ is already computed.
Then, after an edge insertion or an edge deletion,
the $1\pm\epsilon$ approximation
can be updated in $O(rd)$ time.
\end{theorem}
\begin{proof}
After one of the above-mentioned update operations,
by the {\em edge-update-efficient} property
of $\mathbold M$,
$\mathbold M$ can be updated by at most $K$
pairs of update vectors
for the revised graph.
Given these at most $K$ pairs of update vectors
and $(\mathbold S \cdot \mathbold M)^{\dag}$ of
the graph before the update operation,
we want to compute $(\mathbold S \cdot \mathbold M)^{\dag}$
of the revised graph.
Since the number of columns
and the number of rows of $\mathbold{M}$
do not change, the sketching matrix $\mathbold S$
does not change, too.
We have a sequence of at most $K$ {\em rank-$1$} updates
$\mathbold{M^{k+1}}=\mathbold{M^k}+\mathbold{c^k} \cdot {\mathbold{d^k}}^*$,
$1 \leq k < K$,
where
$\mathbold{c^k}$ and ${\mathbold{d^k}}$ are a pair of update vectors,
$\mathbold{M^1}=\mathbold{M}$
and $\mathbold{M^K}$ is the correct matrix embedding of $G$
after the update operation.
After each {\em rank-$1$} update $\mathbold{M^{k+1}}=\mathbold{M^k}+\mathbold{c^k} \cdot {\mathbold{d^k}}^*$,
\begin{itemize}
\item
given the matrix $\mathbold{S} \cdot \mathbold{M^{k}}$,
we first compute
$\mathbold{S} \cdot \mathbold{c^k} \cdot \mathbold{d^k}^*$
and then, we compute
$\mathbold{S} \cdot \mathbold{M^{k+1}}$
by computing the matrix summation
$\mathbold{S} \cdot \mathbold{M^{k}} +
\mathbold{S} \cdot \mathbold{c^{k}} \cdot \mathbold{d^{k}}^*$.
Note that
$\mathbold{S} \cdot \mathbold{c^k} \cdot \mathbold{d^k}^*$
can be computed in
$O\left( r d \right)$ time, as follows.
First, we compute $\mathbold{S} \cdot \mathbold{c^{k}}$
by taking into account only the $i^{th}$ column of $\mathbold S$,
where $i$ is the sole nonzero entry of $\mathbold{c^{k}}$.
The result is a vector
$\mathbold{s^{k}}$ of size
$r$.
Second, we compute the vector product
$\mathbold{s^{k}} \cdot \mathbold{d^k}^*$,
which can be done in
$O\left( rd \right)$ time.
\item
then, we exploit the algorithm of Meyer~\cite{Generalized_Inverse_2}
that given a $n_1 \times n_2$ matrix $\mathbold{A}$ and its Moore-Penrose pseudoinverse
$\mathbold{A}^{\dag}$
and a pair of update vectors $\mathbold{c}$ and $\mathbold{d}$,
computes the Moore-Penrose pseudoinverse of $(\mathbold{A}+\mathbold{c} \cdot \mathbold{d}^*)$,
in $O(n_1 n_2)$ time\footnote{Instead of
Meyer's algorithm~\cite{Generalized_Inverse_2},
we can also use the general reduction technique of
van~den~Brand~\cite{DBLP:journals/corr/abs-2010-13888} that maintains several operations (including pseudoinverse) on dynamic matrices,
by maintaining only one specific matrix inverse.
}.
Here, our matrix $\mathbold A$ is
$\mathbold S \cdot \mathbold M$
which is a $r \times d$ matrix,
therefore updating
$(\mathbold S \cdot \mathbold M)^{\dag}$
for a given pair of update vectors
will take
$O\left( rd \right)$ time.
\end{itemize}
Therefore and after repeating this procedure for at most $K$ times,
we can compute the Moore-Penrose pseudoinverse of
$\mathbold S \cdot \mathbold M$
for the updated graph
in $O\left(Krd\right)=
O\left(rd\right)$ time.
In the end, multiplication of the updated
$(\mathbold S \cdot \mathbold M)^{\dag}$
with $(\mathbold S \cdot \mathbold b)$
yields the approximate solution
of the updated graph,
which can be done in $O(rd)$ time.
\end{proof}


\subsection{Addressing the limitations}
\label{sec:limitations}

The first intrinsic limitation of
randomized Hadamard transform is that
the number of rows in $\mathbold M$, i.e., $n$,
must be a power of $2$.
This implies that we should always have
a power of $2$ nodes in the graph.
When applying
randomized Hadamard transform to matrices,
this issue is addressed by
concatenating a zero matrix to the main matrix
that makes its size a power of $2$
\cite{DBLP:conf/nips/LuDFU13,DBLP:journals/corr/abs-1805-05421}.
We can follow a similar strategy for graphs.
More precisely,
if during pre-processing the number of rows
of $\mathbold M$ is less than a power of $2$,
we pad it with zeros up to the next larger power of $2$.
This might be seen as adding {\em isolated nodes}
to the graph, with measured values $0$,
to make its size a power of $2$.
The second intrinsic limitation of
randomized Hadamard transform is that
$\mathbold M$ must be a {\em full rank} matrix.
However, this might not be a serious problem
as many real-world matrices
have a full rank.

The next restriction is that
the $n \times d$ matrix embedding $\mathbold M$
must satisfy two properties.
First, $d \ll n$, because otherwise,
randomized Hadamard transform will not be efficient.
Second, it must be {\em edge-update-efficient}.
In the following, first in
Definition~\ref{def:m-embedding}, we present
a matrix embedding defined based on the
$d$ closest nodes of each node,
where $d$ can be arbitrarily small
(we consider it as a small constant).
So it satisfies the first property.
Then in Theorem~\ref{theorem:edge_update_efficient},
we prove that it is an
{\em edge-update-efficient} matrix embedding.
For the sake of simplicity, we assume that $G$
is undirected.
The extension of the results to directed graphs
is straightforward.

\begin{definition}
\label{def:m-embedding}
For each node $v$ in a graph $G$, we define
its vector embedding as a vector consisting of $d$ nodes
of $G$ that have the smallest distances to $v$,
and call it the $d$-nearest neighborhood of $v$.
If there are several such subsets of $V(G)$,
we choose an arbitrary one.
We define matrix embedding $\mathbold M$ of $G$
as a $n \times d$ matrix
whose $i^{th}$ row is the vector embedding of
node $i$.
\end{definition}

\begin{lemma}
\label{lemma:m-nn}
If node $u$ is reachable from node $v$
(i.e., there is a path from $v$ to $u$)
but their distance is larger than $d$,
$u$ cannot be in $d$-nearest neighborhood of $v$.
\end{lemma}
\begin{proof}
If $u$ and $v$ are connected by a path but $dist(u,v)>d$,
there exist at least $d$ nodes in the graph, such that
their distances to $v$ are less than $dist(u,v)$.
Therefore, $u$ is not in $d$-nearest neighborhood of $v$.
\end{proof}

\begin{lemma}
\label{lemma:m-closest}
If an edge is inserted between nodes $u$ and $v$
of graph $G$,
vector embeddings of at most $O(d^d)$ nodes
of $G$ may change.
Furthermore, each vector embedding that must be revised,
can be updated in $O(d^2)$ time.
\end{lemma}
\begin{proof}
First, we determine those nodes that after inserting
an edge between $u$ and $v$,
{\em may} have a change in their
$d$-nearest neighborhood.
Let $Q$ denote the set of such nodes.
Nodes $u$ and $v$ belong to $Q$.
Also, those nodes that have already node $u$ (resp. node $v$)
in their $d$-nearest neighborhood,
after inserting an edge between $u$ and $v$
may find $v$ (resp. $u$) and some other nodes
in their $d$-nearest neighborhood.
Lets focus on finding those nodes that
have already $u$ in their $d$-nearest neighborhood,
and may have $v$ in their $d$-neighborhood
after the edge insertion
(finding those nodes that may have $u$
in their $d$-neighborhood after the edge insertion
can be done in a similar way).
To do so, we conduct a breadth-first search (BFS)
from $v$ on the \underline{updated graph}.
We use the following pruning/stopping criterion's:
\begin{itemize}
\item
at the first level, among all neighbors of $v$,
we meet only $u$.
The reason is that we are interested in finding those nodes
that have a shortest path to $v$ passing over $u$.
\item
in other levels,
if a node $x$ has a degree greater than $d$,
$v$ cannot be in the $d$-nearest neighborhood
of any of its adjacent nodes
(and also any node $y$ such that $x$ is on a shortest path between $y$ and $v$). Because the adjacent nodes of $x$
have already at least $d$ nodes that are closer to them than $v$.
\item
if a node $x$ has a distance greater than $d$ from $v$,
as Lemma~\ref{lemma:m-nn} says,
$v$ cannot be in its $d$-nearest neighborhood.
Furthermore, any node $y$ such that
$v$ is on a shortest path from $x$ to $y$ cannot be
in the $d$-nearest neighborhood of $x$. Hence,
those nodes that have a distance greater than $d$ from $v$
should not be traversed during the BFS.
\end{itemize}
As a result and in the end of the traversal,
all the met nodes have a degree at most $d$
and a distance to $v$ at most $d$.
The number of such nodes is at most $O(d^{d})$.

Second, from each node whose vector embedding may require an update,
we conduct a BFS on its first $d$ nodes
to compute its updated embedding.
This can be done in $O(d^2)$ time.
\end{proof}

\begin{lemma}
\label{lemma:m-closest_delete}
If the edge between nodes $u$ and $v$
of a graph $G$ is deleted,
vector embeddings of at most $O(d^d)$ nodes change.
Furthermore, each vector embedding that should be revised,
can be updated in $O(d^2)$ time.
\end{lemma}
\begin{proof}
Our proof is similar to the proof
of Lemma~\ref{lemma:m-closest}.
First, we determine those nodes that after deleting the edge between $u$ and $v$,
{\em may} have a change in their
neighborhood.
Let $Q$ denote the set of such nodes.
Nodes $u$ and $v$ belong to $Q$.
Also, those nodes that have already node $u$
(resp. node $v$)
in their $d$-nearest neighborhood,
after deleting the edge between $u$ and $v$,
may also loose $v$ (resp. $u$) and some other nodes
from their $d$-nearest neighborhood.
Lets focus on finding those nodes that
have already $u$ in their $d$-nearest neighborhood,
and may loose $v$ and some other nodes
from their $d$-neighborhood
(finding those nodes that may loose $u$
from their $d$-neighborhood can be done in a similar way).
We conduct a BFS from $v$
on the graph \underline{before} the edge deletion.
We use the three pruning/stopping criterion's used
in the proof of Lemma~\ref{lemma:m-closest}.
In the end of the traversal,
all the met nodes have a degree at most $d$
and a distance to $v$ at most $d$. The number of such nodes is at most $O(d^{d})$.

Second, from each node whose embedding may require an update, we conduct a BFS on its first $d$ nodes
in the updated graph,
to compute its updated embedding.
This can be done in $O(d^2)$ time.
\end{proof}

\begin{theorem}
\label{theorem:edge_update_efficient}
Assuming that $d$ is a constant,
the matrix embedding $\mathbold M$ defined in
Definition~\ref{def:m-embedding} is
an {\em edge-update-efficient}$_1$ matrix embedding.
\end{theorem}
\begin{proof}
We show that $\mathbold M$ satisfies the conditions
stated in Definition~\ref{def:update_efficient}.
When an edge is inserted/deleted
between nodes $i$ and $j$,
as Lemmas~\ref{lemma:m-closest}
and \ref{lemma:m-closest_delete} say,
vector embeddings of at most $O(d^d)$ nodes change
and it take $O(d^2)$ time to update each vector embedding.
Since $d$ is a constant,
we can consider $d^{d+2}$
as a constant.
For each node $v$
whose vector embedding has been changed,
we define a pair of update vectors
$\mathbold{c}$ and $\mathbold d$
as follows:
$\mathbold d$ contains the new vector embedding
of $v$ minus its old vector embedding; and
the position and the value of the nonzero entry of
$\mathbold c$ are respectively set to $v$ and $1$.
Therefore, the conditions of
Definition~\ref{def:update_efficient} are satisfied and
$\mathbold M$ is an {\em edge-update-efficient}$_1$
matrix embedding.
\end{proof}

\begin{corollary}
\label{corollary:phd}
Suppose that we are given a graph $G$ whose
matrix embedding is defined as
Definition~\ref{def:m-embedding},
with $d$ as a small constant,
and it is a full rank matrix.
Our first randomized algorithm, which is based on
subsampled randomized Hadamard transform,
performs the pre-processing phase in
\begin{equation}
\label{eq:phd_pre_time}
O\left( n \log_2 \left( \ln n \ln \ln n + \frac{\ln n}{\epsilon} \right) + \ln n \ln \ln n +\frac{\ln n}{\epsilon} \right)
\end{equation}
time.
Then, after an edge insertion or an edge deletion,
it updates a $1\pm\epsilon$ approximation to
the optimal solution of graph regression in
\begin{equation}
\label{eq:update_time}
O\left(\ln n \ln \ln n + \frac{\ln n}{\epsilon} \right)
\end{equation}
time.
\end{corollary}
\begin{proof}
In Theorem~\ref{theorem:phd_edge},
we conditioned on the existence of
an {\em edge-update-efficient} embedding and showed that
it takes $O(rd)$ time to update the approximate solution.
Then in Theorem~\ref{theorem:edge_update_efficient},
we showed that this matrix embedding does exist.
Therefore and by using the value of $r$ presented in
Theorem~\ref{theorem:phd} and
discarding constants (including $d$),
we obtain the time complexities
stated in the theorem.
\end{proof}

We note that
when computing embeddings for nodes of a graph,
the objective is to map each node to a vector in a {\em low} dimensional space~\cite{chen_wang_wang_kuo_2020,DBLP:journals/tkde/CaiZC18}.
In this way, many embedding computation methods
such those that are based on random-walks~\cite{DBLP:conf/kdd/GroverL16} and deep graph neural networks~\cite{DBLP:conf/nips/HamiltonYL17},
consider only a {\em small} neighborhood for each node.
Therefore, it is reasonable to consider $d$ as a {\em small} constant.
We also note that if the exact algorithm
of \cite{DBLP:journals/corr/abs-1903-10699} uses
the matrix embedding presented in
Definition~\ref{def:m-embedding}, it will yield
a linear time algorithm (in terms of $n$) for updating the solution,
which is considerably worse than the sublinear
update time presented in
Equation~\ref{eq:update_time}.

\section{Dynamic graph regression using \textsf{CountSketch}}
\label{section:countsketch}

In this section, we utilize \textsf{CountSketch}
to develop our second randomized algorithm for the
dynamic graph regression problem.
Unlike our first algorithm,
it works for all the update operations:
i) {\em node insertion}, wherein a node is inserted
into the graph and at most a constant number of edges
are drawn between it and the existing nodes of the graph,
ii) {\em node deletion}, wherein a node that has at most
a constant number of edges,
is deleted from the graph and its incident edges
are deleted, too,
iii) {\em edge deletion} wherein an edge
is deleted from the graph, and
iv) {\em edge insertion} wherein an edge is
inserted into the graph.

We assume that
a $n\times d$ matrix embedding exists which
satisfies the following conditions:
i) $d$ is fixed and does not depend on the number
of data rows $n$ (as a result, by changing the number of
data rows, $d$ does not change),
and ii) the matrix embedding is a \textsf{CUE} embedding.
\textsf{CUE}\footnote{\textsf{CUE} is abbreviation for
\textbf{C}ountSketch-based
\textbf{U}pdate-\textbf{E}fficient matrix embedding.}
characterizes a class of matrix embeddings
for which we can efficiently update the
approximate solution of graph regression,
using \textsf{CountSketch}.
It is less general then
{\em edge-update-efficient} matrix embeddings
presented in Section~\ref{section:randomized},
which can be used for only {\em edge-related} operations.

\begin{definition}
\label{def:cue}
Let $\mathbold{M}$ be a $n \times d$ matrix embedding of a graph $G$
and $f$ be a (complexity) function of $n$ and $d$.
We say $\mathbold{M}$ is \textsf{CUE}$_f$,
iff the following conditions are satisfied:
\begin{enumerate}
\item
if $\mathbold{M}$ and $\mathbold{M'}$ are
correct matrix embeddings before and after
an edge insertion/deletion in the graph,
there exist at most $K$ pairs of vectors $\mathbold{c^k}$ and $\mathbold{d^k}$, with $K$ as a constant,
such that:
$$\mathbold{M'}=\mathbold{M}+\sum_{k=1}^K \left( \mathbold{c^k} \cdot \mathbold{d^k}^* \right).$$
Each vector $\mathbold{d^k}$ has size $d$ and
each vector $\mathbold{c^k}$
has size $n$ wherein only one entry,
whose position is known,
is nonzero.
\item
a node insertion in $G$ results in adding
one row to $\mathbold{M}$
and also (at most) a rank-$K$ update matrix
in $\mathbold{M}$.
\item
deleting a node
from $G$ results in deleting
one row from $\mathbold{M}$
and also (at most) a rank-$K$ update matrix
in $\mathbold{M}$.
\item
after any update operation in $G$, it is feasible to compute
all pairs of update vectors in $O(f(n,d))$ time.
\end{enumerate}
When $f$ is clear from the context, we drop it.
\end{definition}


Similar to the case of subsampled randomized Hadamard transform,
during the pre-processing phase of
our \textsf{CountSketch}-based algorithm and
for a given $\epsilon$,
first we generate
a $q \times d$ matrix $\mathbold S$,
as defined in Section~\ref{section:sketch}.
Then we calculate $\mathbold{M'} = \mathbold S \cdot \mathbold M$
and $\mathbold{b'} = \mathbold S \cdot \mathbold b$.
Finally, we compute $\mathbold{M'}^{\dag}$ and
$\mathbold{M'}^{\dag} \cdot \mathbold{b'}$.
Time complexity of the
procedure is given in Theorem~\ref{theorem:countsketch}.
In the following, first in Section~\ref{sec:countsketch_alg}
we discuss how the approximate solution is updated,
after an update operation.
Then, in Section~\ref{sec:cue_existence} we discuss
the existence of a \textsf{CUE} matrix embedding.

\subsection{The update algorithm}
\label{sec:countsketch_alg}

In this section, we assume that we are given a matrix $\mathbold M$
that satisfies the two above mentioned conditions
and show, using \textsf{CountSketch}, how the approximate solution is
efficiently updated after an update operation.

\subsubsection{Edge insertion/deletion}
\label{sec:edge}

In this section, we assume that the update operation is either an edge insertion
or an edge deletion.
Then, we show that the approximate solution can be updated in $O(qd)$ time.

\begin{theorem}
\label{theorem:countsketch_edge}
Assume that $\mathbold M$ is a $n \times d$
\textsf{CUE} matrix embedding of graph $G$.
Suppose also that
using a $ q \times n$ \textsf{CountSketch}
$\mathbold S$ with $q$ defined in
Equation~\ref{eq:count_sketch_sample_no},
a $1\pm \epsilon$ approximation to the solution of
graph regression of $G$ is already computed.
Then, after an edge insertion or an edge deletion,
the approximate solution
can be updated in $O(qd)$ time.
\end{theorem}
\begin{proof}
The proof is similar to the proof of
Theorem~\ref{theorem:phd_edge}.
Since $\mathbold M$ is a
\textsf{CUE} matrix embedding,
after an edge insertion or an edge deletion,
$\mathbold M$ is updated by at most $K$
pairs of update vectors.
Since the number of columns of $\mathbold{M}$
does not change, matrix $\mathbold S$
does not change, too.
Therefore, we have a sequence of at most $K$ {\em rank-$1$} updates
$\mathbold{M^{k+1}}=\mathbold{M^k}+\mathbold{c^k} \cdot {\mathbold{d^k}}^*$,
$1 \leq k < K$,
where
$\mathbold{c^k}$ and ${\mathbold{d^k}}$ are a pair of update vectors,
$\mathbold{M^1}=\mathbold{M}$
and $\mathbold{M^K}$ is the correct matrix embedding of $G$
after the update operation.
After each {\em rank-$1$} update $\mathbold{M^{k+1}}=\mathbold{M^k}+\mathbold{c^k} \cdot {\mathbold{d^k}}^*$,
given the matrix $\mathbold{S} \cdot \mathbold{M^{k}}$,
similar to the proof of Theorem~\ref{theorem:phd_edge},
we can compute
$\mathbold{S} \cdot \mathbold{M^{k+1}}$
in $O(qd)$ time.
Then we can use
Meyer's algorithm~\cite{Generalized_Inverse_2}
to update $(\mathbold S \cdot \mathbold M)^{\dag}$,
for a given pair of update vectors,
in $O\left( qd \right)$ time.

After repeating this procedure for
at most $K$ times,
we can compute the Moore-Penrose pseudoinverse of
$\mathbold S \cdot \mathbold M$
for the updated graph in
$O\left(qd\right)$ time.
Finally, multiplication of the updated
$(\mathbold S \cdot \mathbold M)^{\dag}$
with $(\mathbold S \cdot \mathbold b)$
can generate, in $O(qd)$ time,
the approximate solution.
\end{proof}

\subsubsection{Node insertion}
\label{section:node_insertion}

In this section, we assume that the update operation is
a {\em node insertion} and show
how the approximate solution can be effectively updated.

\begin{theorem}
\label{theorem_node_insertion}
Let $\mathbold M$ be a $n \times d$
\textsf{CUE} matrix embedding of graph $G$.
Suppose that using a $q \times n$
\textsf{CountSketch} $\mathbold S$
with $q$ defined in Equation~\ref{eq:count_sketch_sample_no},
a $1\pm\epsilon$ approximation to
the solution of graph regression of $G$ is already computed.
Then, after inserting a node into $G$,
the approximate solution can be updated in $O(qd)$ time.
\end{theorem}

\begin{proof}
After inserting a node into the graph,
we need to revise matrices
$\mathbold S$ and $\mathbold M$.
Matrix $\mathbold M$ is revised because
we need to add to $\mathbold M$
the row corresponding to the new node.
Matrix $\mathbold S$ is revised because
its number of columns is a function of
the number of rows of $\mathbold M$.
Therefore and as a result of a node insertion,
we add a new column to $\mathbold S$ and
we choose a row uniformly at random as its nonzero element.
Let $i$ be the index of this nonzero row.
To update $\mathbold S \cdot \mathbold M$
with respect to this change,
we add to each entry $j$ of the $i^{th}$
row of $\mathbold S \cdot \mathbold M$
the value of the $j^{th}$ entry of the
last row of $\mathbold M$.
This can be done in $O(d)$ time.
Furthermore, by the \textsf{CUE}
property of $\mathbold M$,
as a result of this node insertion,
the vector embeddings of the other nodes
change by at most $K$ pairs of update vectors.
Since $q$ and $d$ do not change,
the size of matrix $\mathbold S \cdot \mathbold M$
does not change, too.
Updating $\mathbold S \cdot \mathbold M$
with respect to these
at most $K$ pairs of update vectors
can be done in $O(qd)$ time (as described in the
proofs of Theorems~\ref{theorem:phd_edge} and
\ref{theorem:countsketch_edge}).

To update $(\mathbold S \cdot \mathbold M)^{\dag}$
with respect to the changes in
$\mathbold S \cdot \mathbold M$,
we can exploit the algorithm of Meyer~\cite{Generalized_Inverse_2}.
Since the changes in $i^{th}$ row of
$\mathbold S \cdot \mathbold M$
can be expressed in terms of a pair of update vectors,
$(\mathbold S \cdot \mathbold M)^{\dag}$
can be updated with respect to them in $O(qd)$ time.
Furthermore,
for each of at most $K$ pairs of update vectors,
we can use the algorithm of Meyer~\cite{Generalized_Inverse_2}
to update $(\mathbold S \cdot \mathbold M)^{\dag}$
in $O(qd)$ time.

After node insertion,
we need also to append the measured value of the new
node to the bottom of $\mathbold b$ and then,
update $\mathbold S \cdot \mathbold b$
(with respect to the revised $\mathbold S$).
To update $\mathbold S \cdot \mathbold b$,
it is sufficient to add the measured value of the new node to the $i^{th}$ entry of $\mathbold S \cdot \mathbold b$
($i$ is the nonzero row of the new column of
the updated $\mathbold S$).
In the end, a naive multiplication of the updated
$(\mathbold S \cdot \mathbold M)^{\dag}$
with the updated $\mathbold S \cdot \mathbold b$ gives
the approximate solution for the updated graph,
and it can be done in $O(qd)$ time.
\end{proof}

\subsubsection{Node deletion}
\label{section:node_deletion}

In this section, we assume that the update operation is
{\em node deletion},
and show, in Theorem~\ref{theorem:node_deletion},
how the approximate solution is effectively updated.

\begin{theorem}
\label{theorem:node_deletion}
Let $\mathbold M$ be a $n \times d$ \textsf{CUE}
matrix embedding of graph $G$.
Suppose that using a $q \times n$ \textsf{CountSketch}
$\mathbold S$ with $q$ defined in
Equation~\ref{eq:count_sketch_sample_no},
a $1\pm\epsilon$ approximation to the solution of graph regression of $G$
is already computed.
Then, after deleting a node from $G$,
the approximate solution can be updated in $O(qd)$ time.
\end{theorem}
\begin{proof}
After deleting a node from the graph,
we need to revise matrices
$\mathbold S$ and $\mathbold M$.
Matrix $\mathbold M$ is revised because
we need to delete from it
the row corresponding to the deleted node.
Matrix $\mathbold S$ is revised because
we should delete from it the
the column corresponding to the deleted node.
Let $i$ be the index of this nonzero row.
To update $\mathbold S \cdot \mathbold M$
with respect to these changes,
we subtract from each entry $j$ of the $i^{th}$
row of $\mathbold S \cdot \mathbold M$
the value of $\mathbold M[q,j]$.
This can be done in $O(d)$ time.
Furthermore, by the \textsf{CUE}
property of $\mathbold M$,
as a result of this node deletion,
the vector embeddings of the other nodes
may change by at most $K$ pairs of update vectors.
Matrix $\mathbold S \cdot \mathbold M$
can be updated with respect to these changes in
$O(qd)$ time.
Since $q$ and $d$ do not change,
the size of matrix $\mathbold S \cdot \mathbold M$
does not change, too.

To update $(\mathbold S \cdot \mathbold M)^{\dag}$
with respect to these changes in
$\mathbold S \cdot \mathbold M$,
we can again exploit the algorithm of Meyer~\cite{Generalized_Inverse_2}.
Therefore,
since the changes in the $i^{th}$ row of
$\mathbold S \cdot \mathbold M$
can be expressed in terms of a pair of update vectors,
$(\mathbold S \cdot \mathbold M)^{\dag}$
can be updated with respect to them in $O(qd)$ time.
Also, for each of at most $K$ pairs of update vectors,
we can use the algorithm of Meyer~\cite{Generalized_Inverse_2}
to update $(\mathbold S \cdot \mathbold M)^{\dag}$
in $O(qd)$ time.

After node deletion,
we need also to delete the measured value of the deleted
node from $\mathbold b$ and then,
update $\mathbold S \cdot \mathbold b$.
To update $\mathbold S \cdot \mathbold b$,
it is sufficient to subtract the measured value of the deleted node from the the $i^{th}$ entry of $\mathbold S \cdot \mathbold b$,
where $i$ is the nonzero entry of the deleted column.
In the end, a naive multiplication of the updated
$(\mathbold S \cdot \mathbold M)^{\dag}$
with the updated $\mathbold S \cdot \mathbold b$ yields
the approximate solution of the updated graph
and it can be done in $O(qd)$ time.
\end{proof}

\subsection{Existence of a \textsf{CUE} matrix embedding}
\label{sec:cue_existence}

In this section, we show
that the $d$-nearest neighborhood
vector embedding presented in Section~\ref{sec:limitations}
satisfies all the conditions we are looking for.
First of all, in this embedding
$d$ is a small constant and does not depend on $n$.
Second, in Theorem~\ref{theorem:cue}
we show that it is \textsf{CUE}\footnote{More than
these two conditions and similar to our first randomized algorithm, here our matrix embedding must
be a full rank matrix.
}.

\begin{theorem}
\label{theorem:cue}
Assuming that $d$ is a constant,
the matrix embedding $\mathbold M$ defined in
Definition~\ref{def:m-embedding} of Section~\ref{sec:limitations}
is \textsf{CUE}$_1$.
\end{theorem}
\begin{proof}
We shall show that $\mathbold M$ satisfies all the conditions
stated in Definition~\ref{def:cue} (for $f=1$).
\begin{enumerate}
\item
When an edge is inserted/deleted between
nodes $i$ and $j$,
in a way similar to the proof of
Theorem~\ref{theorem:edge_update_efficient},
we can show that condition~(1) of
Definition~\ref{def:cue} is satisfied.
\item
When a new node $i$ is added to $G$,
we add a new row for it in $\mathbold M$,
which contains its $d$ closest neighbors.
Furthermore, since at most a constant number $C$ of
edges are added between $i$ and existing nodes in $G$
and each edge insertion may change vector embeddings of
at most $O(d^d)$ nodes, vector embeddings
of at most $O(Cd^d)$ nodes change,
which can be seen as a constant $K$.
Therefore and similar to the previous case,
condition~(2) of
Definition~\ref{def:update_efficient} is satisfied.
\item
When we delete a node from $G$,
we delete its corresponding row from $\mathbold M$.
Furthermore, since the deleted node
may have at most a constant number $C$ of edges
(that are deleted too),
and each edge deletion may change vector embeddings of
at most $O(d^d)$ nodes, vector embeddings
of at most $O(Cd^d)$ nodes change,
which is a constant $K$.
Hence and similar to the previous case,
condition~(3) of
Definition~\ref{def:update_efficient} is satisfied.
\item
For all the update operations,
each pair of update vectors
$\mathbold c$ and $\mathbold d$
can be computed in $O(d^2)$ time.
As a result, condition (4) of
Definition~\ref{def:update_efficient} is satisfied.
\end{enumerate}
\end{proof}

\begin{corollary}
\label{corollary:count_sketch}
Suppose that we are given a graph $G$ whose
matrix embedding is defined as
Definition~\ref{def:m-embedding},
with $d$ as a constant,
and it is a full rank matrix.
Using a \textsf{CountSketch} as the sketching matrix,
we can perform the pre-processing phase in
\begin{equation}
\label{eq:count_sketch_pre_time}
O\left(
n + \epsilon^{-2}\log^7(1/\epsilon)
\right)
\end{equation}
time.
Then, after a node insertion
or a node deletion or an edge insertion or an edge deletion,
we can update the $1\pm\epsilon$ approximation to the solution
of graph regression in
$O\left( \frac{1}{\epsilon^2} \log^6 (1/\epsilon)  \right)$
time.
\end{corollary}

\begin{proof}
In Theorem~\ref{theorem:countsketch_edge},
we conditioned on the existence of
a \textsf{CUE} matrix embedding and showed that
it takes $O(qd)$ time to update the approximate solution.
Then in Theorem~\ref{theorem:edge_update_efficient},
we showed the existence of this matrix embedding.
As a result and by replacing $q$
with its value
defined in
Equation~\ref{eq:count_sketch_sample_no}
and discarding all constants (including $d$),
we obtain the time complexities
stated in the theorem.
\end{proof}


\section{Discussion}
\label{sec:discussion}

When $d \ll n$, both of our randomized algorithms
outperform the exact algorithm of \cite{DBLP:journals/corr/abs-1903-10699}, in terms of pre-processing
and update times.
However, we shall also compare the two randomized
algorithms against each other.

\begin{itemize}
\item
Suppose that our randomized algorithms use
$d$-nearest neighborhood matrix embedding and
we discard the terms
$\ln \ln n$ and $\log^6(1/\epsilon)$
from the update time complexities
(due to having terms such as
$\ln n$ and $\epsilon^{-2}$).
Under these assumptions,
update time complexities of the first and second algorithms become
$O\left( \frac{\ln n}{\epsilon}\right)$ and
$O\left( \epsilon^{-2} \right)$, respectively.
Hence, if $\ln n \geq \epsilon^{-1}$,
the second algorithm finds a smaller
update time,
otherwise the first algorithm outperforms
the second algorithm in terms of update time.

Note that
in the general form and without
relaying on any specific matrix embedding,
our first algorithm
updates the $1\pm\epsilon$ approximation
in a sublinear time
in terms of $n$ (Theorem~\ref{theorem:phd_edge} of
Section~\ref{sec:phd_edge}).
However, when we use \textsf{CountSketch},
the update time becomes independent of $n$
(Theorems~\ref{theorem:countsketch_edge}, \ref{theorem_node_insertion} and \ref{theorem:node_deletion} of Section~\ref{sec:countsketch_alg}).
In particular, if we consider $d$ and $\epsilon$
as constants,
while update time of our first algorithm
is a sublinear of $n$ (it still depends on $n$),
our second algorithm updates the $1\pm\epsilon$
approximation in a constant time.
As a result and in addition to the useful
{\em sparsity} property of
\textsf{CountSketch} \cite{Clarkson:2017:LAR:3038256.3019134},
its another interesting property
discussed in this paper
is its {\em constant update time}
for all the update operations
node insertion, node deletion, edge insertion and edge deletion.

\item
Similar to the case of update times,
we may simplify pre-processing times
by assuming that the algorithms use
the $d$-nearest neighborhood matrix embedding.
We discard the terms
$\log \log n$ and $\log^7(1/\epsilon)$
from Equations~\ref{eq:phd_pre_time} and
\ref{eq:count_sketch_pre_time}.
Then, the
pre-processing time complexities of the
first and second algorithms become
$O\left(n+ \frac{\ln n}{\epsilon} \right)$ and
$O\left(n+\epsilon^{-2}\right)$, respectively.
Therefore
if $\epsilon^{-1} > \sqrt{n}$,
the first algorithm finds a smaller pre-processing time
than the second algorithm.
\end{itemize}

\section{Conclusion}
\label{sec:conclusion}

In this paper, we
presented sublinear update time
randomized algorithms for dynamic graph regression.
For a $n\times d$
efficiently updatable matrix embedding $\mathbold M$
where $d \ll n$,
our first algorithm is based on
subsampled randomized Hadamard transform and
supports
edge insertion and edge deletion.
It updates a
$1\pm \epsilon$ approximation of the optimal solution
in $O(rd)$ time, where
$r$ is a sublinear of $n$.
Our second algorithm
is based on \textsf{CountSketch}
and supports
edge insertion, edge deletion,
node insertion and node deletion.
It updates a $1\pm\epsilon$ approximation of
the optimal solution
in $O(qd)$ time,
where $q=O\left(\frac{d^2}{\epsilon^2} \log^6(d/\epsilon) \right)$.


\vskip 0.2in

\bibliographystyle{plain}
\bibliography{allpapers}   

\end{document}